\documentclass[twoside]{article}
\usepackage[accepted]{aistats2015}
\usepackage{hyperref}
\usepackage{natbib,url}
\usepackage{amsmath,amssymb,amsthm}
\usepackage{multirow}
\usepackage{graphicx}
\usepackage{subfigure}
\usepackage{wrapfig}
\usepackage{bbm}
\usepackage{epstopdf}

\theoremstyle{plain} \newtheorem{thm}{Theorem}
 \newtheorem{lem}[thm]{Lemma}

\theoremstyle{definition} 

\theoremstyle{remark} 

\newtheorem{lemma}{Lemma}

  \newenvironment{Proof}{\noindent{\bf Proof} \ }{\QED}\smallskip
\newcommand\QED{\newline \rightline{$\blacksquare$} \bigskip}

  \newenvironment{Proof thm4}{\noindent{\bf Proof of Theorem 4} \ }{\QED}\smallskip

  \newenvironment{Proof thm5}{\noindent{\bf Proof of Theorem 5} \ }{\QED}\smallskip

  \newenvironment{Proof thm6}{\noindent{\bf Proof of Theorem 6} \ }{\QED}\smallskip

\newcommand\Rp{\mathcal{R}_{L,P,\lambda}}
\newcommand\RT{\mathcal{R}_{L,T}}
\newcommand\RP{\mathcal{R}_{L,P}}

\newcommand\h{\mathcal{H}}
\newcommand\e{\epsilon_n}
\newcommand\X{\mathcal{X}}
\newcommand\x{\mathbf{X}}
\newcommand\F{\mathcal{F}}

\begin{document}

%

%

\twocolumn[

\aistatstitle{Learning Efficient Anomaly Detectors from $K$-NN Graphs}

\aistatsauthor{ Jing Qian \And Jonathan Root \And Venkatesh Saligrama }

\aistatsaddress{ Boston University \And Boston University \And Boston University } ]

\begin{abstract}
We propose a non-parametric anomaly detection algorithm for high dimensional
data. We score each datapoint by its average $K$-NN distance, and rank them 
accordingly. We then train limited complexity models to imitate these
scores based on the max-margin learning-to-rank framework. A test-point is declared as an anomaly
at $\alpha$-false  alarm level if the predicted score is in the $\alpha$-percentile.
The resulting anomaly detector is shown to be asymptotically optimal in that for
any false alarm rate $\alpha$, its decision region converges to the $\alpha$-percentile
minimum volume level set of the unknown underlying density. In addition, we test
both the statistical performance and computational efficiency of our
algorithm on a number of synthetic and real-data experiments. Our results
demonstrate the superiority of our algorithm over existing $K$-NN based
anomaly detection algorithms, with significant computational savings.
\end{abstract}

\section{Introduction}\label{sec:intro}

Anomaly detection is the problem of identifying statistically significant deviations in data from expected normal behavior. It has found wide applications
in many areas such as credit card fraud detection, intrusion detection for cyber security, sensor networks and video surveillance \citep{ref:anomaly_detection_survey,ref:AD_survey_hodge}.

In classical parametric methods \citep{ref:para_1993} for anomaly detection, we assume the existence of a family of functions characterizing the nominal density 
(the test data consists of examples belonging to two classes--nominal
and anomalous).
Parameters are then estimated from training data by minimizing a loss function. While these methods provide a statistically justifiable solution when the assumptions hold true, they are likely to suffer from model mismatch, and lead to poor performance.

We focus on the non-parametric approach, with a view towards
minimum volume (MV) set estimation. Given $\alpha\in (0,1)$,
the MV approach attempts to find the set of minimum volume which
has probability mass at least $1-\alpha$ with respect to the unknown
sample probability distribution. Then given a new test point, it is
declared to be consistent with the data if it lies in this MV set.

Approaches to the MV set estimation problem include
estimating density level sets \citep{ref:levelset_mvset_2003,ref:Cuevas2003}
or estimating the boundary of the MV set \citep{ref:MV_2006,ref:MV_2010}.
However, these approaches suffer from high sample complexity,
and therefore are statistically unstable using high dimensional data.
The authors of \citep{ref:Manqi2009} score each test point using the
$K$-NN distance. Scores turn out to yield empirical estimates
of the volume of minimum volume level sets containing the test
point, and avoids computing any high dimensional quantities.
The papers \citep{ref:GEM_2006,ref:knn_2011}
also take a $K$-NN based approach to MV set anomaly detection.
The second paper \citep{ref:knn_2011} improves upon the
computational performance of \citep{ref:GEM_2006}. However,
the test stage runtime of \citep{ref:knn_2011}  is of order
$O(dn)$, $d$ being the ambient dimension and $n$ the 
sample size. 
The test stage runtime of  \citep{ref:Manqi2009} is of order $O(dn^2+n^2\log(n))$. 

Computational inefficiencies of these $K$-NN based anomaly detection methods suggests that a different approach based on distance-based (DB) outlier methods (see \citep{orair} and references therein) could possibly be leveraged in this context. 
DB methods primarily focus on the {\it computational} issue of identifying a pre-specified number of $L$ points (outliers) with largest $K$-NN distances in a database. Outliers are identified by pruning examples with small $K$-NN distance. This works particularly well for small $L$. 

In contrast, for anomaly detection, we not only need an efficient scheme but also one that takes training data (containing no anomalies) and generalizes well in terms of AUC criterion on test-data where the number of anomalies is unknown. We need schemes that predict ``anomalousness'' for test-instances in order to adapt to any false-alarm-level and to characterize AUCs. One possible way to leverage DB methods is to estimate anomaly scores based only on the $L$ identified outliers but this scheme generally has poor AUC performance if there are a sizable fraction of anomalies. In this context \citep{ref:isolation_forest,ref:massAD,ref:knn_2011} propose to utilize ORCA \citep{orca}. ORCA is a well-known {\it ranking DB} method that provides intermediate estimates for every instance in addition to the $L$ outliers. They show that while for small $L$ ORCA is highly efficient its AUC performance is poor. For large $L$ ORCA produces low but somewhat meaningful AUCs but can be computationally inefficient. A basic reason for this AUC gap is that although such {\it rank-based} DB techniques provide intermediate KNN estimates \& outlier scores that can possibly be leveraged, these estimates/scores are often too unreliable for anomaly detection purposes. Recently, \citep{lsh} have considered strategies based on LSH to further speed up rank based DB methods.  Our perspective is that this direction is somewhat complementary. Indeed, we could also employ Kernel-LSH \citep{klsh} in our setting to further speed up our computation. 


In this paper, we propose a ranking based algorithm which retains
the statistical complexity of existing $K$-NN work, but with far
superior computational performance.
Using scores based on average $K$NN distance, we learn a functional
predictor through the pair-wise learning-to-rank framework, to predict $p$-value scores. This predictor is then used to
generalize over unseen examples. The test time of our algorithm is
of order $O(ds)$, where $s$ is the complexity of our model.


The rest of the paper is organized as follows. In Section 2 we introduce the problem setting and the motivation. Detailed algorithms are described in Section 3 and 4. The asymptotic and finite-sample analyses are provided in Section 5. Synthetic and real experiments are reported in Section 6. 
\vspace{-0.1in}
\section{Problem Setting \& Motivation}
Let $\bold{x}=\lbrace x_1, x_2, ..., x_n\rbrace $ be a given set of nominal $d$-dimensional data points. We assume $\bold{x}$ to be sampled i.i.d from an unknown density $f_0$ with compact support in $\mathbb{R}^d$. The problem is to assume a new data point,
$\eta\in \mathbb{R}^d$, is given, and test whether $\eta$ follows the distribution
of $\bold{x}$. If $f$ denotes the density of this new (random) data point, then the set-up
is summarized in the following hypothesis test:
\[
H_0: f=f_0 \;\;\;\;\; \text{vs.} \;\;\;\;\; H_1: f\neq f_0.
\]
We look for a functional $D:\mathbb{R}^d\to \mathbb{R}$ such that $D(\eta)>0\implies$ $\eta$ nominal.
Given such a $D$, we define its corresponding acceptance region by $A= \{ x : D(x)>0\}$.
We will see below that $D$ can be defined by the $p$-value.

Given a prescribed significance level (false alarm level) $\alpha\in (0,1)$, we require the probability that $\eta$ {\it does not} deviate from the nominal ($\eta \in A$), given $H_0$, to be bounded below by $1-\alpha$.
We denote this distribution by $P$ (sometimes written $P(\text{not } H_1 | H_0)$):
\[
P(A)=\int_A f_0(x) \; dx \geq 1-\alpha.
\]
Said another way, the probability that $\eta$ {\it does} deviate from the nominal,
given $H_0$, should fall under the specified significance level $\alpha$ (i.e.
$1-P(A)=P( H_1 | H_0) \leq \alpha$).
At the same time, the false negative, $\int_A f(x) \; dx $, must be minimized.
Note that the false negative is the probability of the event $\eta \in A$, given $H_1$.
We assume $f$ to be bounded
above by a constant $C$, in which case $\int_A f(x) \; dx \leq C\cdot \lambda(A)$,
where $\lambda$ is Lebesgue measure on $\mathbb{R}^d$. The problem of finding
the most suitable acceptance region, $A$, can therefore
be formulated as finding the following minimum volume set:
\begin{equation}\label{MV}
U_{1-\alpha}:=\arg\min_{A} \left\{ \lambda(A) : \int_A f_0(x) \; dx \geq 1-\alpha \right\}.
\end{equation}
In words, we seek a set $A$ which captures at least a fraction $1-\alpha$ of
the probability mass, of minimum volume.

 \section{Score Functions Based on K-NNG} \label{sec:scorefunc}
  In this section, we briefly review an algorithm using score functions based on
  nearest neighbor graphs for determining minimum volume sets \citep{ref:Manqi2009, ref:Jing2012}.
  Given a test point $\eta \sim f$, define the $p$-value
  of $\eta$ by
  \[
   p(\eta) := P\left( x : f_0(x)\leq f_0(\eta)\right)= \int_{\{x : f_0(x)\leq f_0(\eta)\}}f_0(x)\; dx.
   \]
  Then, assuming technical conditions on the density $f_0$ \citep{ref:Manqi2009},
   it can be shown that $p$ defines the minimum volume set:
  \[
  U_{1-\alpha} = \{x : p(x)\geq \alpha\}.
  \]
  Thus if we know $p$, we know the minimum volume set,
  and we can declare anomaly simply by checking whether or not
  $p(\eta) < \alpha$. However, $p$
  is based on information from the unknown density $f_0$, hence we
  must estimate $p$.

  Set $d(x,y)$ to be the Euclidean metric on $\mathbb{R}^d$.
  Given a point $x\in \mathbb{R}^d$,  we form its associated $K$ nearest
  neighbor graph (K-NNG), relative to $\bold{x}$,
  by connecting it to the $K$ closest
  points in $\bold{x}\setminus \{x\}$. Let $D_{(i)}(x)$
  denote the distance from $x$ to its $i$th nearest neighbor
  in $\bold{x}\setminus \{x\}$.
  Set
  \begin{equation}\label{$K$-NNstat}
  G_{\bold{x}}(x) = \frac{1}{K}\sum_{j=1}^KD_{(j)}(x).
  \end{equation}
  Now define the following score function:
  \begin{equation}\label{estimate_p}
  R_n(\eta) := \frac{1}{n} \sum_{i=1}^n \textbf{1}_{\{G_{\bold{x}}(\eta) < G_{\bold{x}}(x_i)\}}
  \end{equation}
  This function measures the relative concentration of point $\eta$ compared to
  the training set.
In \citep{ref:Jing2012},
given a pre-defined significance level $\alpha$ (e.g. 0.05), they declare $\eta$
to be anomalous if $R_n(\eta) \leq \alpha$. This choice is motivated
by its close connection to multivariate $p$-values. Indeed,
it is shown in \citep{ ref:Jing2012} that this score function is an asymptotically consistent estimator of the $p$-value:
\[
\lim_{n\to \infty}   R_n(\eta) = p(\eta) \;\;\; \text{a.s.}
\]
This result is attractive from a statistical viewpoint, however the test-time complexity of the
$K$-NN distance statistic grows as $O(dn)$. 
This can be prohibitive for real-time applications. Thus we are compelled to
learn a score function respecting the $K$-NN distance statistic, but with
significant computational savings. 
This is achieved by mapping the data set $\bold{x}$ into a reproducing kernel Hilbert
space (RKHS), $H$, with kernel $k$ and inner product $\langle \cdot, \cdot\rangle$.
We denote by $\Phi$ the mapping $\mathbb{R}^d\to H$,
defined by $\Phi(x_i) =k(x_i,\cdot)$. We then optimally learn a ranker $g\in H$
based on the ordered pair-wise ranking information,
\[
\{(i,j) : G_{\bold{x}}(x_i) > G_{\bold{x}}(x_j)\}
\]
and construct the scoring function as
\begin{equation}\label{estimate_pn}
\hat{R}_n(\eta) := \frac{1}{n} \sum_{i=1}^n \textbf{1}_{\{\langle g,\Phi(\eta)\rangle < \langle g, \Phi(x_i)\rangle\}}.
\end{equation}
It turns out that $\hat{R}$ is an asymptotic estimator of the $p$-value (see Section
\ref{sec:analysis})
and thus we will say a test point $\eta$ is anomalous if $\hat{R}(\eta)\leq \alpha$.



\section{Anomaly Detection Algorithm}\label{sec:main_algo}
In this section we describe our rank-based anomaly detection algorithm (RankAD), and discuss several of its properties and advantages. 

\noindent\rule[0.5ex]{\linewidth}{1pt}
\noindent\textbf{Algorithm 1: RankAD Algorithm}

\noindent\rule[0.5ex]{\linewidth}{1pt}
\noindent\textbf{1. Input:}
\noindent Nominal training data $\bold{x}={\lbrace x_1, x_2, ..., x_n \rbrace}$, desired false alarm level $\alpha$, and test point $\eta$.

\noindent\textbf{2. Training Stage:}

\noindent(a) Calculate $K$th nearest neighbor distances $G_{\bold{x}}(x_i)$, and calculate
$R_n(x_i)$ for each nominal sample $x_i$, using Eq.(\ref{$K$-NNstat}) and Eq.(\ref{estimate_p}).

\noindent(b) Quantize $\{R_n(x_i),\,i=1,2,...,n\}$ uniformly into $m$ levels: $r_q(x_i) \in \lbrace 1,2,..., m \rbrace$. Generate preference pairs $(i, j)$ whenever their quantized levels are different: $r_q(x_i) > r_q(x_j)$.

\noindent(c) 
Set $\mathcal{P}=\{(i,j) : r_q(x_i)> r_q(x_j)\}$. Solve:
\begin{eqnarray}\label{eq:ranksvm_standard}
  \min_{g,\xi_{ij}}: \, && \,\, \frac{1}{2} ||g||^2 + C \sum_{(i,j)\in\mathcal{P}} \xi_{ij} \\
\nonumber
  s.t. \, &&  \,\, \langle g,\, \Phi(x_i)-\Phi(x_j) \rangle \geq 1 - \xi_{ij}, \,\,\,\, \forall (i,j)\in \mathcal{P}  \\
\nonumber
          &&  \,\,  \xi_{ij} \geq 0
\end{eqnarray}

\noindent(d) Let $\hat{g}$ denote the minimizer. Compute and sort: $\hat{g}(\cdot)=\langle \hat{g}, \Phi(\cdot)\rangle$ on $\bold{x}={\lbrace x_1, x_2, ..., x_n \rbrace}$.

\noindent\textbf{3. Testing Stage:}

\noindent(a) Evaluate $\hat{g}(\eta)$ for test point $\eta$.

\noindent(b) Compute the score: $\hat{R}_n(\eta) = \frac{1}{n}\sum^n_{i=1} \textbf{1}_{\{ \hat{g}(\eta) < \hat{g}(x_i) \}}$. This can be done through a binary search over sorted $\{ \hat{g}(x_i),i=1,...,n \}$.

\noindent(c) Declare $\eta$ as anomalous if $\hat{R}_n(\eta) \leq \alpha$.

\noindent\rule[0.5ex]{\linewidth}{1pt}

\noindent{\bf Remark 1:}
The standard learning-to-rank setup \citep{ref:ranksvm} is to assume non-noisy input pairs.
Our algorithm is based on noisy inputs, where the noise is characterized by an unknown, high-dimensional distribution.
Yet we are still able to show the asymptotic consistency of the obtained ranker in Sec.\ref{sec:analysis}.

\noindent{\bf Remark 2:} 
For the learning-to-rank step Eq.(\ref{eq:ranksvm_standard}), we equip the RKHS $H$ with the RBF kernel $k(x,x') = \exp\left( -\dfrac{\lVert x-x'\rVert^2}{\sigma^2} \right)$. The algorithm parameter $C$ and RBF kernel bandwidth $\sigma$ can be selected through cross validation, since this step is a supervised learning procedure based on input pairs. We use cross validation and adopt the weighted pairwise disagreement loss (WPDL) from \citep{ref:RDPS2012} for this purpose.

\noindent{\bf Remark 3:} 
The number of quantization levels, $m$, impacts training complexity as well as performance.
When $m=n$, all $n \choose 2$ preference pairs are generated. This scenario has the highest training complexity. Furthermore, large $m$ tends to more closely follow rankings obtained from $K$-NN distances, which may or may not be desirable. $K$-NN distances can be noisy for small training data sizes.
While this raises the question of choosing $m$, we observe that setting $m$ to be $3\sim5$ works fairly well in practice. We fix $m=3$ in all of our experiments in Sec.\ref{sec:exp}.
$m=2$ is insufficient to allow flexible false alarm control, as will be demonstrated next.

\noindent{\bf Remark 4:} 
Let us mention the connection with ranking SVM. Ranking SVM is an algorithm
for the learning-to-rank problem, whose goal is to rank unseen objects based
on given training data and their corresponding orderings. Our novelty lies
in building a connection between learning-to-rank and anomaly detection: \\
(1) While there is no such natural ``input ordering'' in anomaly detection, we 
create this order on training samples through their $K$-NN scores. \\
(2)
When we apply our detector on an unseen object it produces a score
that approximates the unseen object's $p$-value. We theoretically
justify this linkage, namely, our predictions fall in the right
quantile (Theorem \ref{thm:AD_finite_sample}). We also empirically
show test-stage computational benefits. 


\subsection{False alarm control}
In this section we illustrate through a toy example how our learning method approximates minimum volume sets.
We consider how different levels of quantization impact level sets. We will show that for appropriately chosen quantization levels
our algorithm is able to simultaneously approximate multiple level sets. In Section~\ref{sec:analysis} we show that the normalized score Eq.(\ref{estimate_pn}), takes values in $[0,1]$, and converges to the $p$-value function. Therefore we get a handle on the false alarm rate. So null hypothesis can be rejected at different levels simply by thresholding $\hat{R}_n(\eta)$.

{\bf Toy Example:} \\
We present a simple example in Fig. 1
to demonstrate this point. The nominal density $ f \sim 0.5 \mathcal{N} \left(
\left[ 4; 1 \right],
0.5 I
 \right)
+
0.5 \mathcal{N} \left(
\left[ 4 ; -1 \right],
0.5 I
 \right)
$.
We first consider single-bit quantization ($m=2$) using RBF kernels ($\sigma=1.5$) trained with pairwise preferences between $p$-values above and below $3$\%. 
This yields a decision function $\hat{g}_2(\cdot)$. The standard way is to claim anomaly when $\hat{g}_2(x)<0$, corresponding to the outmost orange curve in (a). We then plot different level curves by varying $c>0$ for $\hat{g}_2(x)=c$, which appear to be scaled versions  of the orange curve. 
While this quantization appears to work reasonably for $\alpha$-level sets with $\alpha=3$\%, for a different desired $\alpha$-level, the algorithm would have to retrain with new preference pairs. On the other hand, we also train rankAD with $m=3$ (uniform quantization) and obtain the ranker $\hat{g}_3(\cdot)$. We then vary $c$ for $\hat{g}_3(x)=c$ to obtain various level curves shown in (b), all of which surprisingly approximate the corresponding density level sets well.
We notice a significant difference between the level sets generated with 3 quantization levels in comparison to those generated for two-level quantization.
In the appendix we show that $\hat{g}(x)$ asymptotically preserves the ordering of the density, and from this conclude that our score function $\hat{R}_n(x)$ approximates multiple density level sets ($p$-values). Also see Section \ref{sec:analysis} for a discussion of this.
However in our experiments it turns out that we just need $m=3$ quantization levels instead of $m=n$ ($n \choose 2$ pairs) to achieve flexible false alarm control and do not need any re-training.

\begin{figure*}[htb!]
\vspace{-0.1in}
\begin{centering}
\begin{minipage}[t]{.44\textwidth}
\includegraphics[width = 1\textwidth]{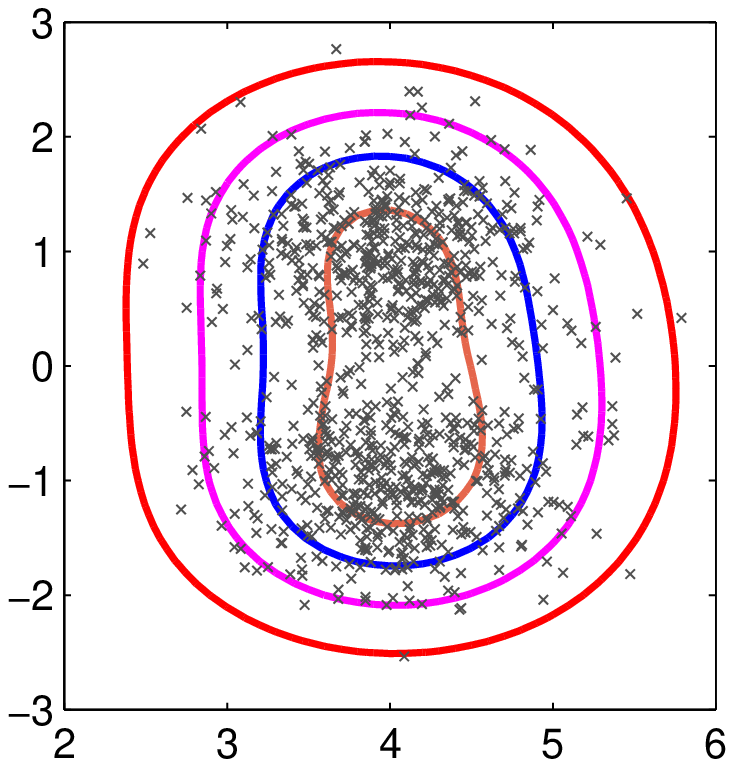}
\makebox[6 cm]{(a) Level curves ($m=2$)}\medskip
\end{minipage}
\begin{minipage}[t]{.44\textwidth}
\includegraphics[width = 1\textwidth]{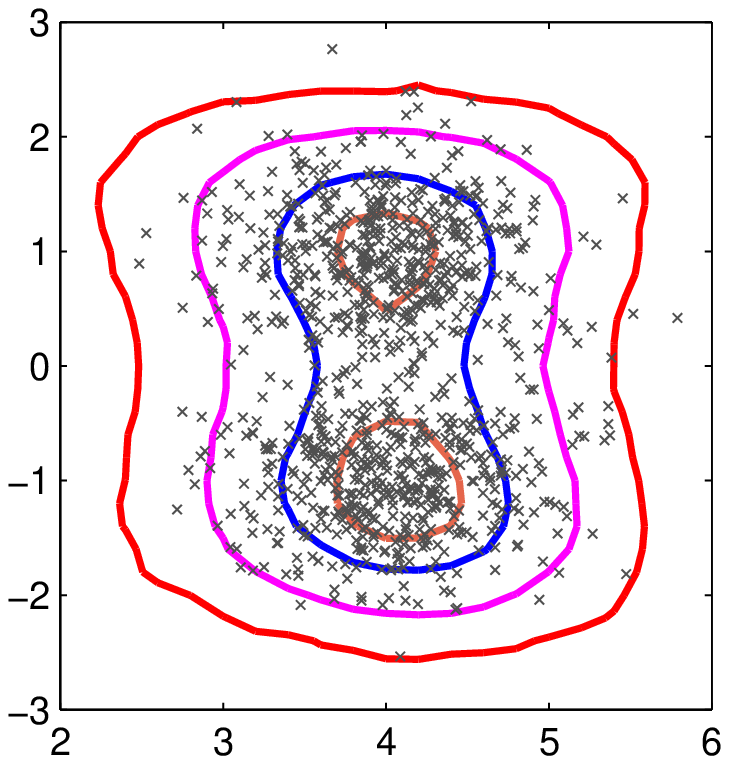}
\makebox[6 cm]{(b) Level curves ($m=3$)}\medskip
\end{minipage}
\vspace{-0.1in}
\caption{\small Level curves of rankAD for different quantization levels. 1000 i.i.d. samples are drawn from a 2-component Gaussian mixture density. Left figure(a) depicts performance with single-bit quantization ($m=2$). To learn rankAD we quantized preference pairs at 3\% and $\sigma=1.5$ in our RBF kernel. Right figure(b) shows rankAD with 3-levels of quantization and $\sigma=1.5$.  (a) shows level curves obtained by varying the offset $c$ for $\hat{g}_2(x)=c$. Only the outmost curve ($c=0$) approximates the oracle density level set well while  the inner curves ($c>0$) appear to be scaled versions of outermost curve. (b) shows level curves obtained by varying $c$ for $\hat{g}_3(x)=c$. Interestingly we observe that the inner most curve approximates peaks of the mixture density. }
\end{centering}
\vspace{-0.1in}
\label{fig:level_curve}
\end{figure*}

\subsection{Time Complexity}
For training, the rank computation step requires computing all pair-wise distances among nominal points $O(dn^2)$, followed by sorting for each point $O(n^2\log n)$.
So the training stage has the total time complexity $O(n^2(d+\log n)+T)$, where $T$ denotes the time of the pair-wise learning-to-rank algorithm.
At test stage, our algorithm only evaluates $\hat{g}(\eta)$ on $\eta$ and does a binary search among $\hat{g}(x_1),\ldots,\hat{g}(x_n)$. The complexity is $O(ds +\log n)$, where $s$ is the number of support vectors. This has some similarities with one-class SVM where the complexity scales with the number of support vectors \citep{ref:oc_svm2001}. Note that in contrast nearest neighbor-based algorithms, K-LPE, aK-LPE or BP-$K$-NNG \citep{ref:Manqi2009,ref:Jing2012,ref:knn_2011}, require $O(nd)$ for testing one point. It is worth noting that $s\leq n$ comes from the ``support pairs'' within the input preference pair set. 
Practically we observe that for most data sets $s$ is much smaller than $n$ in the experiment section, leading to significantly reduced test time compared to aK-LPE, as shown in Table.1.
It is worth mentioning that distributed techniques for speeding up computation of $K$-NN distances \citep{ref:Bhaduri11} can be adopted to further reduce test stage time.


\section{Analysis}\label{sec:analysis}
In this section we present the theoretical analysis of our ranking-based anomaly detection approach.

\subsection{Asymptotic Consistency}\label{subsec:AD_asymptotic}


As mentioned earlier in the paper, it is shown in \citep{ ref:Jing2012}  that the average $K$-NN distance statistic
converges to the $p$-value function:
\begin{thm}\label{thm:consistency_step1}
With $K=O(n^{0.5})$, we have
$\lim_{n\to\infty}R_n(\eta) = p(\eta)$. 
\end{thm}
The goal of our rankAD algorithm is to learn the ordering of the $p$-value. 
This theorem therefore guarantees that asymptotically, the preference pairs generated as input to the rankAD algorithm are reliable. Note that the definition of $G$ in  \citep{ref:Jing2012} is slightly
different than the one  given in equation (\ref{$K$-NNstat}). However, for our purposes
this difference is not worth detailing. 

What we claim in this paper, and prove in the appendix, is the following consistency
result of our rankAD algorithm. Note that the use of quantization (c.f. Section \ref{sec:main_algo})
 does not affect the
conclusion of this theorem, hence we assume there is none. Indeed, quantization
is a computational tool. From a statistical asymptotic consistency perspective quantization is not an issue. 

\begin{thm}\label{thm:rank_main_consistency}
With $K=O(n^{0.5})$, 
as $n\rightarrow\infty$, $\hat{R}_n(\eta)\rightarrow p(\eta).$
\end{thm}

The difficulty in this theorem arises from the fact that the score, $\hat{R}_n(\eta)$,
is based on the ranker, $\hat{g}$, which is learned from data with high-dimensional
noise. Moreover, the noise is distributed according to an {\it unknown} probability
measure.  For the proof of this theorem, we begin with the law of large numbers.
Suppose for any $n\geq 1$, a function $G$ is found such that
  $f(x_i)<f(x_j)\implies G(x_i)<G(x_j)$. Note that in Section~\ref{sec:scorefunc} we use $K$-NN distance surrogates which reverses the order but the effect is the same and should not cause any confusion. Then it can be shown that
\[
\frac{1}{n} \sum_{i=1}^n \textbf{1}_{\{G(x_i) < G(\eta)\}} \rightarrow p(\eta).
\]
Thus we wish to prove that the output of our rankAD algorithm
is such a function.

The first step in our proof is to show that the solution to our rankAD
algorithm, $\hat{g}$, is consistent   \citep{ref:Steinwart2001}.
 Fix an RKHS $H$ on the input
space $X\subset \mathbb{R}^d$ with RBF kernel $k$. 
We denote by $L$ the hinge loss. We may
write $\hat{g}$ as the solution to the following regularized minimization 
problem,
\[
\hat{g} = \arg \min_{f\in H} \mathcal{R}_{L,T}(f) +\lambda_n\lVert f \rVert_H^2,
\]
where 
$\mathcal{R}_{L,T}(f)=\frac{1}{n^2}\sum_{i,j}L(f(x_i)-f(x_j)).$
$T$ denotes the pairs from the sample $\bold{x}=\{x_1,\dots,x_n\}$, so this is a loss with respect
to the empirical measure. The expected risk is denoted 
\[
\mathcal{R}_{L,P}(f) = E_{\bold{x}}[\mathcal{R}_{L,T}(f)].
\]
Then consistency means that, under appropriate conditions as
$\lambda_n\to 0$ and $n\to \infty$ (see appendix), we have 
\begin{equation}\label{consistent}
E_{\bold{x}}[ \mathcal{R}_{L,T}(\hat{g})] \to \min_{f\in H} \mathcal{R}_{L,P}(f).
\end{equation}
The proof of this claim requires a concentration of measure result
relating $\mathcal{R}_{L,T}(f)$ to its expectation, $\mathcal{R}_{L,P}(f)$,
uniformly over $f\in H$. The argument follows closely that made in
\citep{ref:Smale2001}, except we make use of McDiarmid's inequality.

Finally we show that if $\hat{g}$ satisfies (\ref{consistent}), then 
it ranks samples according to their density: $f(x_i)>f(x_j)\implies
\hat{g}(x_i)>\hat{g}(x_j)$. 

\subsection{Finite-Sample Generalization Result}\label{subsec:finite_sample}
Based on a sample $\{x_1,\ldots,x_n\}$, our approach learns a ranker $g_{n}$, and computes the values $g_{n}(x_1),\ldots,g_{n}(x_n)$. Let
$g_n^{(1)}\leq g_n^{(2)}\leq \cdots \leq g_n^{(n)}$ be the ordered permutation of these values.
For a test point $\eta$, we evaluate $g_{n}(\eta)$ and compute $\hat{R}_n(\eta)$.
For a prescribed false alarm level $\alpha$, we define the decision region for claiming anomaly by
\begin{eqnarray*}
  R_\alpha &=& \{ x: \, \hat{R}_n(x) \leq \alpha  \} \\
   &=&  \{ x: \, \sum_{j=1}^{n} \textbf{1}_{\{ g_n(x)\leq g_n(x_j) \}} \leq \alpha n \}  \\
   &=&  \{ x: \, g_n(x) < g_n^{ \lceil \alpha n \rceil } \}
\end{eqnarray*}
where $\lceil \alpha n \rceil$ denotes the ceiling integer of $\alpha n$.

We give a finite-sample bound on the probability that a newly drawn nominal point $\eta$ lies in $R_\alpha$. In the following Theorem, $\mathcal{F}$ denotes a
real-valued function class of kernel based linear functions
equipped with the $\ell_{\infty}$ norm over
a finite sample $\mathbf{x}=\{x_1,\dots,x_n\}$:
\[
\lVert f \rVert _{\ell_{\infty}^{\mathbf{x}}} =\max_{x\in \mathbf{x}} |f(x)|.
\]
Note that $\mathcal{F}$ contain solutions
to an SVM-type problem, so we assume the output of our rankAD
algorithm, $g_n$, is an element of $\mathcal{F}$.
We let $\mathcal{N}(\gamma,\mathcal{F},n)$ denote the covering
number of $\mathcal{F}$ with respect to this norm (see appendix
for details).

\begin{thm}\label{thm:AD_finite_sample}
Fix a distribution $P$ on $\mathbb{R}^d$ and suppose $x_1,\dots,x_n$ are generated iid from $P$.
For $g\in\mathcal{F}$ let $g^{(1)}\leq g^{(2)}\leq \cdots \leq g^{(n)}$ be the ordered permutation of
$g(x_1),\dots,g(x_n)$. Then for such an $n$-sample, with probability $1-\delta$,
for any $g\in \mathcal{F}$, $1\leq m \leq n$ and sufficiently small $\gamma>0$,
\[
P\left\{x: g(x)< g^{(m)} - 2\gamma\right\} \leq \frac{m-1}{n} + \epsilon(n,k,\delta),
\]
where $\epsilon(n,k,\delta)=\frac{2}{n}(k+\log\frac{n}{\delta})$,
 $k=\lceil{\log\mathcal{N}(\gamma,\mathcal{F},2n)}\rceil$.
\end{thm}

\subsubsection*{Remarks}
\vspace{-0.1in}
\noindent
(1) To interpret the theorem notice that the LHS is precisely the probability that a test point drawn from the nominal distribution has a score below the $\alpha\approx \frac{m-1}{n}$ percentile. We see that this probability is bounded from above by $\alpha$ plus an error term that asymptotically approaches zero. This theorem is true irrespective of $\alpha$ and so we have shown that we can simultaneously approximate multiple level sets. \\
(2)
A similar inequality holds for the event giving a lower bound on $g(x)$.
However, let us emphasize that lower bounds are not meaningful for
our context.
The ranks $g^{(1)}\leq g^{(2)}\leq\dots\leq
g^{(n)}$ are sorted in increasing order. A smaller $g(x)$ signifies
that $x$ is more of an outlier. Points below the lowest rank $g^{(1)}$
correspond to the most extreme outliers.


\section{Experiments}
\label{sec:exp}
In this section, we carry out point-wise anomaly detection experiments on synthetic and real-world data sets. We compare our ranking-based approach against density-based methods BP-$K$-NNG \citep{ref:knn_2011} and aK-LPE \citep{ref:Jing2012}, and two other state-of-art methods based on random sub-sampling, isolated forest \citep{ref:isolation_forest} (iForest) and massAD \citep{ref:massAD}.
One-class SVM \citep{ref:oc_svm2001} is included as a baseline.


\subsection{Implementation Details}
In our simulations, the Euclidean distance is used as distance metric for all candidate methods. For one-class SVM the lib-SVM codes \citep{ref:libsvm}  are used. The algorithm parameter and the RBF kernel parameter for one-class SVM are set using the same configuration as in \citep{ref:massAD}.
For iForest and massAD, we use the codes from the websites of the authors, with the same configuration as in \citep{ref:massAD}.
For aK-LPE we use the average $k$-NN distance Eq.(\ref{$K$-NNstat}) with fixed $k=20$ since this appears to work better than the actual $K$-NN distance of \citep{ref:Manqi2009}. Note that this is also suggested by the convergence analysis in Thm~1  \citep{ref:Jing2012}. For BP-$K$-NNG, the same $k$ is used and other parameters are set according to \citep{ref:knn_2011}.

For our rankAD approach we follow the steps described in Algorithm 1. We first calculate the ranks $R_n(x_i)$ of nominal points according to Eq.(3) based on a$K$-LPE.
We then quantize $R_n(x_i)$ uniformly into $m$=3 levels $r_q(x_i)\in\{1,2,3\}$ and generate pairs $(i, j)\in\mathcal{P}$ whenever $r_q(x_i)>r_q(x_j)$.
We adapt the routine from \citep{ref:ranksvm_chapelle} and extend it to a kernelized version for the learning-to-rank step Eq.(\ref{eq:ranksvm_standard}).
The trained ranker is then adopted in Eq.(4) for test stage prediction.
We point out some implementation details of our approach as follows.

\noindent
{\it (1) Resampling:} We follow \citep{ref:Jing2012} and adopt the U-statistic based resampling to compute aK-LPE ranks. We randomly split the data into two equal parts and use one part as ``nearest neighbors'' to calculate the ranks (Eq.(\ref{$K$-NNstat},~\ref{estimate_p})) for the other part and vice versa. Final ranks are averaged over 20 times of resampling.

\noindent
{\it (2) Quantization levels \& K-NN} For real experiments with 2000 nominal training points, we fix $k=20$ and $m=3$. These values are based on noting that the detection performance does not degrade significantly with smaller quantization levels for synthetic data. The $k$ parameter in $K$-NN is chosen to be 20 and is based on Theorem~\ref{thm:consistency_step1} and results from synthetic experiments (see below).

 \noindent      
{\it (3) Cross Validation using pairwise disagreement loss:} For the rank-SVM step we use a 4-fold cross validation to choose the parameters $C$ and $\sigma$. We vary $C\in \{0.001,0.003,0.01,\dots,300,1000\}$, and the RBF kernel parameter $\sigma \in \Sigma = \{2^i\tilde{D}_K,\,i=-10,-9,\dots,9,10\}$, where $\tilde{D}_K$ is the average $20$-NN distance over nominal samples.
      The pair-wise disagreement indicator loss is adopted from \citep{ref:RDPS2012} for evaluating rankers on the input pairs:
  \begin{equation*}
    L(f) = \sum_{(i,j)\in \mathcal{P}} \textbf{1}_{\{ f(x_i) < f(x_j) \}} 
  \end{equation*}  
%
%
%
%
Reported AUC performances are averaged over 5 runs.


\subsection{Synthetic Data sets}
%
%
We first apply our method to a Gaussian toy problem, where the nominal density is: \[ f_0 \sim 0.2 \mathcal{N} \left(
\left[ 5; 0 \right],
\left[ 1, 0 ; 0, 9 \right]
 \right)
+
0.8 \mathcal{N} \left(
\left[ -5 ; 0 \right],
\left[ 9 , 0; 0 , 1 \right]
 \right).
\]
Anomaly follows the uniform distribution within $\{(x,y):\,\,-18\leq x\leq 18, -18\leq y\leq 18\}$. The goal here is to understand the impact of different parameters ($k$-NN parameter and quantization level) used by RankAD. %
Fig.2 shows the level curves for the estimated ranks on the test data. As indicated by the asymptotic consistency (Thm.2) and the finite sample analysis (Thm.3), the empirical level curves of rankAD approximate the level sets of the underlying density quite well.
\begin{figure}[htbp]\label{fig:gaussian}
\vspace{-0.15in}
\begin{centering}
\begin{minipage}[t]{.45\textwidth}
\includegraphics[width = 1\textwidth]{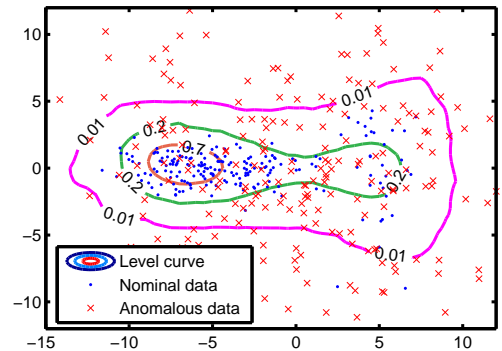}
\end{minipage}
\vspace{-0.1in}
\caption{\small Level sets for the estimated ranks. 600 training points are used for training.}
\end{centering}
\vspace{-0.15in}
\end{figure}
%
We vary $k$ and $m$ and evaluate the AUC performances of our approach shown in Table \ref{tab:sensitivity}.
The Bayesian AUC is obtained by thresholding the likelihood ratio using the generative densities.
From Table \ref{tab:sensitivity} we see the detection performance is quite insensitive to the $k$-NN parameter and the quantization level parameter $m$, and for this simple synthetic example is close to Bayesian performance.

\begin{table}[!htbp]
\vspace{-0.1in}
\caption{ \small AUC performances of Bayesian detector, aK-LPE, and rankAD with different values of $k$ and $m$. 600 training points are used for training. For test 500 nominal and 1000 anomalous points are used. \label{tab:sensitivity} }
\begin{center}
\begin{tabular}{|c||c|c|c|c|}
  \hline
  AUC       & k=5  &  k=10  & k=20 & k=40 \\
  \hline\hline
  m=3       & 0.9206  &  0.9200  &  0.9223  &  0.9210 \\    \hline
  m=5       & 0.9234  &  0.9243  &  0.9247  &  0.9255 \\    \hline
  m=7       & 0.9226  &  0.9228  &  0.9234  &  0.9213 \\    \hline
  m=10      & 0.9201  &  0.9208  &  0.9244  &  0.9196 \\    \hline\hline
  aK-LPE    & 0.9192  &  0.9251  &  0.9244  &  0.9228 \\    \hline
  Bayesian  & 0.9290  &  0.9290  &  0.9290  &  0.9290 \\    \hline
\end{tabular}
\vspace*{-0.2in}
\end{center}
\end{table}
\begin{table}[!htbp]
\caption{ \small Data characteristics of the data sets used in experiments. $N$ is the total number of instances. $d$ the dimension of data. The percentage in brackets indicates the percentage of anomalies among total instances. \label{tab:data_sets} }
\begin{center}
\begin{tabular}{|c||c|c|c|}
  \hline
  data sets     & $N$  &  $d$  & anomaly class  \\
  \hline\hline
  Annthyroid     & 6832  &  6  &  classes 1,2 \\    \hline
  Forest Cover   & 286048  &  10  &  class 4 vs. class 2  \\ \hline
  HTTP           & 567497  &  3   &  attack  \\ \hline
  Mamography     & 11183   &  6   &  class 1  \\ \hline
  Mulcross       & 262144  &  4   &  2 clusters  \\ \hline
  Satellite      & 6435    &  36  &  3 smallest classes  \\ \hline
  Shuttle        & 49097   &  9   &  classes 2,3,5,6,7  \\ \hline
  SMTP           & 95156  &  3   &  attack   \\
  \hline
\end{tabular}
\end{center}
\vspace{-10pt}
\end{table}
\subsection{Real-world data sets}
We conduct experiments on several real data sets used in \citep{ref:isolation_forest} and \citep{ref:massAD}, including 2 network intrusion data sets HTTP and SMTP from \citep{ref:Yamanishi00}, Annthyroid, Forest Cover Type, Satellite, Shuttle from UCI repository \citep{ref:UCI}, Mammography and Mulcross from \citep{ref:Rocke96}. Table \ref{tab:data_sets} illustrates the characteristics of these data sets.

\begin{table*}[t]
\vspace{-0.2in}
\caption{Anomaly detection AUC performance and test stage time of various methods.}
\begin{center}
\begin{tabular}{|c|c||c|c|c|c|c|c|}
  \hline
  \multicolumn{2}{|c||}{Data Sets} &   rankAD  &   one-class svm  &   BP-$K$-NNG   & aK-LPE & iForest & massAD \\
  \hline\hline
  \multirow{8}{*}{AUC}
  &    Annthyroid        & 0.844         & 0.681   & 0.823   &  0.753  & {\bf 0.856}   & 0.789 \\
  &   Forest Cover      & {\bf 0.932}   & 0.869   & 0.889   &  0.876  & 0.853         & 0.895 \\
  &    HTTP              & {\bf 0.999}   & 0.998   & 0.995   & {\bf 0.999}  & 0.986         & 0.995 \\
  &    Mamography        & {\bf 0.909}   & 0.863   & 0.886   &  0.879  & 0.891         & 0.701 \\
  &    Mulcross          & {\bf 0.998}   & 0.970   & 0.994   & {\bf 0.998}  & 0.971         & 0.998 \\
  &    Satellite         & {\bf 0.885}   & 0.774   & 0.872   &  0.884  & 0.812         & 0.692 \\
  &    Shuttle           & {\bf 0.996}   & 0.975   & 0.985   &  0.995  & 0.992         & 0.992 \\
  &   SMTP              & {\bf 0.934}   & 0.751   & 0.902   &  0.900  & 0.869         & 0.859 \\
  \hline
  \multirow{8}{*}{test time}
    & Annthyroid        & 0.338  & 0.281 & 2.171	& 2.173  & 1.384  & 0.030 \\
    & Forest Cover      & 1.748  & 1.638 & 8.185	& 13.41  & 7.239  & 0.483 \\
    & HTTP              & 0.187  & 0.376 & 2.391	& 11.04  & 5.657  & 0.384 \\
    & Mamography        & 0.237  & 0.223 & 0.981	& 1.443  & 1.721  & 0.044 \\
    & Mulcross          & 2.732  & 2.272 & 8.772 	& 13.75  & 7.864  & 0.559 \\
    & Satellite         & 0.393  & 0.355 & 0.976 	& 1.199  & 1.435  & 0.030 \\
    & Shuttle           & 1.317  & 1.318 & 6.404 	& 7.169  & 4.301  & 0.186 \\
    & SMTP              & 1.116  & 1.105 & 7.912 	& 11.76  & 5.924  & 0.411 \\
  \hline
\end{tabular}
\end{center}
\label{tab:real_AUC}
\vspace{-10pt}
\end{table*}

We randomly sample 2000 nominal points for training. The rest of the nominal data and all of the anomalous data are held for testing. Due to memory limit, at most 80000 nominal points are used at test time. The time for testing all test points and the AUC performance are reported in Table \ref{tab:real_AUC}.

We observe that while being faster than BP-$K$-NNG, aK-LPE and iForest, and comparable to one-class SVM during test stage, our approach also achieves superior performance for all data sets.
The density based aK-LPE and BP-$K$-NNG has somewhat good performance, but its test-time degrades with training set size.
massAD is very fast at test stage, but has poor performance for several data sets.

{\it One-class SVM Comparison}
The baseline one-class SVM has good test time due to the similar $O(dS_1)$ test stage complexity where $S_1$ denotes the number of support vectors.
However, its detection performance is pretty poor, because one-class SVM training is in essence approximating one single $\alpha$-percentile density level set. $\alpha$ depends on the parameter of one-class SVM, which essentially controls the fraction of points violating the max-margin constraints \citep{ref:oc_svm2001}.
Decision regions obtained by thresholding with different offsets are simply scaled versions of that particular level set.
Our rankAD approach significantly outperforms one-class SVM, because it has the ability to approximate different density level sets.

{\it aK-LPE \& BP-$K$-NNG Comparison:}
Computationally RankAD significantly outperforms density-based aK-LPE and BP-$K$-NNG, which is not surprising given our discussion in Sec.4.3.
Statistically, RankAD appears to be marginally better than aK-LPE and BP-$K$-NNG for many datasets and this requires more careful reasoning. To evaluate the statistical significance  of the reported test results we note that the number of test samples range from 5000-500000 test samples with at least 500 anomalous points. Consequently, we can bound test-performance to within 2-5\% error with 95\% confidence ($<2\%$ for large datasets and $<5\%$ for the smaller ones (Annthyroid, Mamography, Satellite) ) using standard extension of known results for test-set prediction~\citep{langford05}. After accounting for this confidence RankAD is marginally better than aK-LPE and BP-$K$-NNG statistically. For aK-LPE we use resampling to robustly ranked values (see Sec.~6.1) and for RankAD we use cross-validation (CV) (see Sec.~6.1) for rank prediction. Note that we cannot use CV for tuning predictors for detection because we do not have anomalous data during training. All of these arguments suggests that the regularization step in RankAD results in smoother level sets and better accounts for smoothness of true level sets (also see Fig~\ref{fig:gaussian}) in some cases, unlike NN methods. 

\section{Conclusions}
\label{sec:con}
We presented a novel anomaly detection framework based on combining statistical density information with a discriminative ranking procedure.
Our scheme learns a ranker over all nominal samples based on the $k$-NN distances within the graph constructed from these nominal points. This is achieved through a pair-wise learning-to-rank step, where the inputs are preference pairs $(x_i,x_j)$ and asymptotically models the situation that data point $x_i$ is located in a higher density region relative to $x_j$. 
%
We then show the asymptotic consistency of our approach, which allows for flexible false alarm control during test stage.
We also provide a finite-sample generalization bound on the empirical false alarm rate of our approach.
Experiments on synthetic and real data sets demonstrate our approach has state-of-art statistical performance as well as low test time complexity.



\bibliographystyle{plainnat}
\bibliography{ranksvm}

\noindent
{\bf Acknowledgment:}
\small {This work is supported by the U.S. DHS Grant 2013-ST-061-ED0001 and US NSF under awards 1218992,  1320547 respectively. The views and conclusions contained in this document are those of the authors and should not be interpreted as necessarily representing the official policies, either expressed or implied, of the agencies.
}

\onecolumn
\section*{Appendix: Proofs of Theorems}

For ease of development, let $n=m_1(m_2+1)$, and divide $n$ data points into: $D=D_0 \cup  D_1 \cup ... \cup D_{m_1}$, where $D_0=\{x_1,...,x_{m_1}\}$, and each $D_j, j=1,...,m_1$ involves $m_2$ points. $D_j$ is used to generate the statistic $\eta$ for $u$ and $x_j\in D_0$, for $j=1,...,m_1$. $D_0$ is used to compute the rank of $u$:
\begin{equation}
    R(u) = \frac{1}{m_1}\sum_{j=1}^{m_1} \mathbb{I}_{\{ G(x_j;D_j)>G(u;D_j) \}}
\end{equation}
We provide the proof for the statistic $G(u)$ of the following form 
\begin{eqnarray}
  G(u;D_j) &=& \frac{1}{l}\sum^{l+\lfloor \frac{l}{2} \rfloor}_{i=l-\lfloor \frac{l-1}{2} \rfloor}\left( \frac{l}{i} \right)^{\frac{1}{d}}D_{(i)}(u).
\end{eqnarray}
where $D_{(i)}(u)$ denotes the distance from $u$ to its $i$-th nearest neighbor among $m_2$ points in $D_j$. Practically we can omit the weight and use the average of 1-st to $l$-st nearest neighbor distances as shown in Sec.3.

\textbf{Regularity conditions:} $f(\cdot)$ is continuous and lower-bounded: $f(x) \geq f_{min}>0$. It is smooth, i.e. $||\nabla f(x)||\leq\lambda$, where $\nabla f(x)$ is the gradient of $f(\cdot)$ at $x$. Flat regions are disallowed, i.e. $\forall x \in \mathbb{X}$, $\forall \sigma>0$, $\mathcal{P}\left\{y: |f(y)-f(x)|<\sigma\right\}\leq M\sigma$, where $M$ is a constant.

\section*{Proof of Theorem 1}

The proof involves two steps:
\begin{itemize}
  \item[1.] The expectation of the empirical rank $\mathbb{E}\left[R(u)\right]$ is shown to converge to $p(u)$ as $n\rightarrow\infty$.
  \item[2.] The empirical rank $R(u)$ is shown to concentrate at its expectation as $n\rightarrow\infty$.
\end{itemize}
The first step is shown through Lemma \ref{lem:expectation}. For the second step, notice that the rank $R(u) = \frac{1}{m_1}\sum_{j=1}^{m_1} Y_j$, where $Y_j = \mathbb{I}_{\{ \eta(x_j;D_j)>\eta(u;D_j) \}}$ is independent across different $j$'s, and $Y_j \in [0,1]$. By Hoeffding's inequality, we have:
\begin{equation}
    \mathbb{P}\left( | R(u) - \mathbb{E}\left[R(u)\right] | > \epsilon \right) < 2\exp\left( -2m_1\epsilon^2 \right)
\end{equation}
Combining these two steps finishes the proof.

\begin{lem}\label{lem:expectation}
By choosing $l$ properly, as $m_2\rightarrow\infty$, it follows that,
$$ | \mathbb{E}\left[R(u)\right] - p(u)| \longrightarrow 0$$
\end{lem}
\begin{proof}
Take expectation with respect to $D$:
\begin{eqnarray}
\mathbb{E}_D\left[R(u)\right]
&=&\mathbb{E}_{D\backslash D_0}\left[\mathbb{E}_{D_0}\left[\frac{1}{m_1}\sum_{j=1}^{m_1}
 \mathbb{I}_{\{\eta(u;D_j)<\eta(x_j;D_j)\}}\right]\right]\\
&=&\frac{1}{m_1}\sum_{j=1}^{m_1}\mathbb{E}_{x_j}\left[
\mathbb{E}_{D_j}\left[
\mathbb{I}_{\{\eta(u;D_j)<\eta(x_j;D_j)\}}\right]\right]\\
&=&\mathbb{E}_x\left[\mathcal{P}_{D_1}\left(\eta(u;D_1)<\eta(x;D_1)\right)\right]
\end{eqnarray}
The last equality holds due to the i.i.d symmetry of $\{x_1,...,x_{m_1}\}$ and $D_1,...,D_{m_1}$. We fix both $u$ and $x$ and temporarily discarding $\mathbb{E}_{D_1}$. Let $F_x(y_1,...,y_{m_2})=\eta(x)-\eta(u)$, where $y_1,...,y_{m_2}$ are the $m_2$ points in $D_1$. It follows:
\begin{equation}
    \mathcal{P}_{D_1}\left(\eta(u)<\eta(x)\right)
    =\mathcal{P}_{D_1}\left(F_x(y_1,...,y_{m_2})>0\right)
    =\mathcal{P}_{D_1}\left(F_x-\mathbb{E}F_x>-\mathbb{E}F_x\right).
\end{equation}

To check McDiarmid's requirements, we replace $y_j$ with $y_j'$. It is easily verified that $\forall j=1,...,m_2$,
\begin{equation}\label{equ:mcdiarmid_condition}
    |F_x(y_1,...,y_{m_2})-F_x(y_1,...,y_j',...,y_{m_2})| \leq 2^{\frac{1}{d}}\frac{2C}{l} \leq \frac{4C}{l}
\end{equation}
where $C$ is the diameter of support. Notice despite the fact that $y_1,...,y_{m_2}$ are random vectors we can still apply MeDiarmid's inequality, because according to the form of $\eta$, $F_x(y_1,...,y_{m_2})$ is a function of $m_2$ i.i.d random variables $r_1,...,r_{m_2}$ where $r_i$ is the distance from $x$ to $y_i$.
Therefore if $\mathbb{E}F_x<0$, or $\mathbb{E}\eta(x)<\mathbb{E}\eta(u)$, we have by McDiarmid's inequality,
\begin{equation}
    \mathcal{P}_{D_1}\left(\eta(u)<\eta(x)\right)
    = \mathcal{P}_{D_1}\left( F_x > 0 \right)
    = \mathcal{P}_{D_1}\left( F_x-\mathbb{E}F_x>-\mathbb{E}F_x \right)
    \leq \exp\left(-\frac{(\mathbb{E}F_x)^2 l^2}{8C^2m_2}\right)
\end{equation}
Rewrite the above inequality as:
\begin{equation}\label{equ:bound_no_expectation}
    \mathbb{I}_{\{\mathbb{E}F_x>0\}}-e^{-\frac{(\mathbb{E}F_x)^2 l^2}{8C^2m_2}}
    \leq \mathcal{P}_{D_1}\left( F_x > 0 \right)
    \leq \mathbb{I}_{\{\mathbb{E}F_x>0\}}+e^{-\frac{(\mathbb{E}F_x)^2 l^2}{8C^2m_2}}
\end{equation}
It can be shown that the same inequality holds for $\mathbb{E}F_x>0$, or $\mathbb{E}\eta(x)>\mathbb{E}\eta(u)$. Now we take expectation with respect to $x$:
\begin{equation}\label{equ:bound_with_expectation}
    \mathcal{P}_x\left(\mathbb{E}F_x>0\right)-\mathbb{E}_x\left[e^{-\frac{(\mathbb{E}F_x)^2 l^2}{8C^2m_2}}\right] \leq
    \mathbb{E}\left[\mathcal{P}_{D_1}\left( F_x > 0 \right)\right] \leq \mathcal{P}_x\left(\mathbb{E}F_x>0\right)+\mathbb{E}_x\left[e^{-\frac{(\mathbb{E}F_x)^2 l^2}{8C^2m_2}}\right]
\end{equation}
Divide the support of $x$ into two parts, $\mathbb{X}_1$ and $\mathbb{X}_2$, where $\mathbb{X}_1$ contains those $x$ whose density $f(x)$ is relatively far away from $f(u)$, and $\mathbb{X}_2$ contains those $x$ whose density is close to $f(u)$. We show for $x \in \mathbb{X}_1$, the above exponential term converges to 0 and $\mathcal{P}\left(\mathbb{E}F_x>0\right) = \mathcal{P}_x\left( f(u)>f(x) \right)$, while the rest $x\in\mathbb{X}_2$ has very small measure. Let $A(x)=\left(\frac{k}{f(x) c_d m_2}\right)^{1/d}$. By Lemma \ref{lem:bound_expectation} we have:
\begin{equation}
    | \mathbb{E}\eta(x) - A(x) | \leq \gamma \left(\frac{l}{m_2}\right)^{\frac{1}{d}} A(x)
    \leq \gamma \left(\frac{l}{m_2}\right)^{\frac{1}{d}} \left(\frac{l}{f_{min}c_d m_2}\right)^{\frac{1}{d}}
    =    \left(\frac{\gamma_1}{c_d^{1/d}}\right) \left(\frac{l}{m_2}\right)^{\frac{2}{d}}
\end{equation}
where $\gamma$ denotes the big $O(\cdot)$, and $\gamma_1 = \gamma \left(\frac{1}{f_{min}}\right)^{1/d}$. Applying uniform bound we have:
\begin{equation}
    A(x)-A(u)- 2\left(\frac{\gamma_1}{c_d^{1/d}}\right) \left(\frac{l}{m_2}\right)^{\frac{2}{d}}
    \leq \mathbb{E}\left[\eta(x) - \eta(u)\right]
    \leq A(x)-A(u)+ 2\left(\frac{\gamma_1}{c_d^{1/d}}\right) \left(\frac{l}{m_2}\right)^{\frac{2}{d}}
\end{equation}
Now let $\mathbb{X}_1=\{ x:|f(x)-f(u)|\geq 3\gamma_1 d f_{min}^{\frac{d+1}{d}} \left(\frac{l}{m_2}\right)^{\frac{1}{d}} \}$. For $x\in \mathbb{X}_1$, it can be verified that $|A(x)-A(u)|\geq 3\left(\frac{\gamma_1}{c_d^{1/d}}\right) \left(\frac{l}{m_2}\right)^{\frac{2}{d}}$, or $|\mathbb{E}\left[\eta(x) - \eta(u)\right]| > \left(\frac{\gamma_1}{c_d^{1/d}}\right) \left(\frac{l}{m_2}\right)^{\frac{2}{d}}$, and $\mathbb{I}_{\{f(u)>f(x)\}}=\mathbb{I}_{\{\mathbb{E}\eta(x)>\mathbb{E}\eta(u)\}}$. For the exponential term in Equ.(\ref{equ:bound_no_expectation}) we have:
\begin{equation}
    \exp\left(-\frac{(\mathbb{E}F_x)^2 l^2}{2C^2m_2}\right)
    \leq \exp\left(-\frac{ \gamma_1^2 l^{2+\frac{4}{d}} }{ 8C^2 c_d^{\frac{2}{d}} m_2^{1+\frac{4}{d}} } \right)
\end{equation}
For $x\in \mathbb{X}_2=\{x:|f(x)-f(u)|< 3\gamma_1 d \left(\frac{l}{m_2}\right)^{\frac{1}{d}}f_{min}^{\frac{d+1}{d}} \}$, by the regularity assumption, we have $\mathcal{P}(\mathbb{X}_2)<3M\gamma_1 d \left(\frac{l}{m_2}\right)^{\frac{1}{d}}f_{min}^{\frac{d+1}{d}}$. Combining the two cases into Equ.(\ref{equ:bound_with_expectation}) we have for upper bound:
\begin{eqnarray}
  \mathbb{E}_D\left[R(u)\right]
  &=& \mathbb{E}_x\left[\mathcal{P}_{D_1}\left(\eta(u)<\eta(x)\right)\right] \\
  &=& \int_{\mathbb{X}_1}\mathcal{P}_{D_1}\left(\eta(u)<\eta(x)\right)f(x)dx +  \int_{\mathbb{X}_2}\mathcal{P}_{D_1}\left(\eta(u)<\eta(x)\right)f(x)dx \\
  &\leq& \left( \mathcal{P}_x\left(f(u)>f(x)\right) + \exp\left(-\frac{ \gamma_1^2 l^{2+\frac{4}{d}} }{ 8C^2 c_d^{\frac{1}{d}} m_2^{1+\frac{4}{d}} } \right) \right)\mathcal{P}(x\in \mathbb{X}_1) + \mathcal{P}(x\in \mathbb{X}_2) \\
  &\leq&  \mathcal{P}_x\left(f(u)>f(x)\right) + \exp\left(-\frac{ \gamma_1^2 l^{2+\frac{4}{d}} }{ 8C^2 c_d^{\frac{1}{d}} m_2^{1+\frac{4}{d}} } \right) + 3M\gamma_1 d f_{min}^{\frac{d+1}{d}} \left(\frac{l}{m_2}\right)^{\frac{1}{d}}
\end{eqnarray}
Let $l=m_2^\alpha$ such that $\frac{d+4}{2d+4}<\alpha<1$, and the latter two terms will converge to 0 as $m_2 \rightarrow \infty$. Similar lines hold for the lower bound. The proof is finished.
\end{proof}

\begin{lem}\label{lem:bound_expectation}
Let $A(x)=\left(\frac{l}{m c_d f(x)}\right)^{1/d}$, $\lambda_1 = \frac{\lambda}{f_{min}}\left(\frac{1.5}{c_d f_{min}}\right)^{1/d}$. By choosing $l$ appropriately, the expectation of $l$-NN distance $\mathbb{E}D_{(l)}(x)$ among $m$ points satisfies:
\begin{equation}
    | \mathbb{E}D_{(l)}(x) - A(x) | = O\left( A(x) \lambda_1 \left(\frac{l}{m}\right)^{1/d} \right)
\end{equation}
\end{lem}

\begin{proof}
Denote $r(x,\alpha)=\min\{r:\mathcal{P}\left(B(x,r)\right)\geq \alpha\}$. Let $\delta_m \rightarrow 0$ as $m \rightarrow \infty$, and $0<\delta_{m}<1/2$.
Let $U\sim Bin(m,(1+\delta_m)\frac{l}{m})$ be a binomial random variable, with $\mathbb{E}U = (1+\delta_{m})l$. We have:
\begin{eqnarray*}
  \mathcal{P}\left(D_{(l)}(x)>r(x,(1+\delta_m)\frac{l}{m})\right)
  &=& \mathcal{P}\left(U<l\right) \\
  &=& \mathcal{P}\left(U<\left(1-\frac{\delta_m}{1+\delta_m}\right)(1+\delta_m)l\right) \\
  &\leq& \exp\left(-\frac{\delta_m^2 l}{2(1+\delta_m)}\right)
\end{eqnarray*}
The last inequality holds from Chernoff's bound. Abbreviate $r_1 = r(x,(1+\delta_m)\frac{l}{m})$, and $\mathbb{E}D_{(l)}(x)$ can be bounded as:
\begin{eqnarray*}
  \mathbb{E}D_{(l)}(x)
  &\leq& r_1\left[1-\mathcal{P}\left(D_{(l)}(x)>r_1\right)\right] + C\mathcal{P}\left(D_{(l)}(x)>r_1\right)  \\
  &\leq& r_1 + C \exp\left(-\frac{\delta_m^2 l}{2(1+\delta_m)}\right)
\end{eqnarray*}
where $C$ is the diameter of support. Similarly we can show the lower bound:
\begin{equation*}
    \mathbb{E}D_{(l)}(x) \geq r(x,(1-\delta_m)\frac{l}{m}) - C \exp\left(-\frac{\delta_m^2 l}{2(1-\delta_m)}\right)
\end{equation*}
Consider the upper bound. We relate $r_1$ with $A(x)$. Notice:
\begin{equation*}
  \mathcal{P}\left(B(x,r_1)\right)=(1+\delta_m)\frac{l}{m} \geq c_d r_1^d f_{min}
\end{equation*}
so a fixed but loose upper bound is $r_1 \leq \left(\frac{(1+\delta_m)l}{c_d f_{min} m}\right)^{1/d} = r_{max}$. Assume $l/m$ is sufficiently small so that $r_1$ is sufficiently small. By the smoothness condition, the density within $B(x,r_1)$ is lower-bounded by $f(x)-\lambda r_1$, so we have:
\begin{eqnarray*}
  \mathcal{P}\left(B(x,r_1)\right) &=& (1+\delta_m)\frac{l}{m} \geq c_d r_1^d \left( f(x)-\lambda r_1 \right)\\
  &=& c_d r_1^d f(x)\left( 1-\frac{\lambda}{f(x)}r_1 \right) \\
  &\geq& c_d r_1^d f(x)\left( 1-\frac{\lambda}{f_{min}}r_{max} \right)
\end{eqnarray*}
That is:
\begin{equation}
    r_1 \leq A(x)\left( \frac{1+\delta_m}{1-\frac{\lambda}{f_{min}}r_{max}} \right)^{1/d}
\end{equation}
Insert the expression of $r_{max}$ and set $\lambda_1 = \frac{\lambda}{f_{min}}\left(\frac{1.5}{c_d f_{min}}\right)^{1/d}$, we have:
\begin{eqnarray*}
  \mathbb{E}D_{(l)}(x)-A(x) &\leq& A(x)\left( \left(\frac{1+\delta_m}{1-\lambda_1 \left(\frac{l}{m}\right)^{1/d}}\right)^{1/d} -1 \right) + C \exp\left(-\frac{\delta_m^2 l}{2(1+\delta_m)}\right) \\
  &\leq& A(x)\left( \frac{1+\delta_m}{1-\lambda_1 \left(\frac{l}{m}\right)^{1/d}}-1 \right) + C \exp\left(-\frac{\delta_m^2 l}{2(1+\delta_m)}\right) \\
  &=& A(x)\frac{\delta_m + \lambda_1 \left(\frac{l}{m}\right)^{1/d}}{1-\lambda_1\left(\frac{l}{m}\right)^{1/d}} + C \exp\left(-\frac{\delta_m^2 l}{2(1+\delta_m)}\right) \\
  &=& O\left( A(x) \lambda_1 \left(\frac{l}{m}\right)^{1/d} \right)
\end{eqnarray*}

The last equality holds if we choose $l=m^{\frac{3d+8}{4d+8}}$ and $\delta_m=m^{-\frac{1}{4}}$. Similar lines follow for the lower bound. Combine these two parts and the proof is finished.

\end{proof}

\section*{Proof of Theorem 2}

We fix an RKHS $H$ on the input space $X\subset\mathbb{R}^d$
with an RBF kernel $k$.
Let $\mathbf{x}=\{x_1,\dots,x_n\}$ be a set of objects to be ranked
in $\mathbb{R}^d$ with labels $\mathbf{r}= \{r_1,\dots,r_n\}$.
Here $r_i$ denotes the label of $x_i$, and $r_i\in \mathbb{R}$.
We assume $\mathbf{x}$ to be a random variable distributed according
to $P$, and $\mathbf{r}$ deterministic.   Throughout
$L$ denotes the hinge loss.

The following notation will be useful in the proof of Theorem 2. 
Take $T$ to be the set of pairs derived from $\bold{x}$ and
define the $L$-$risk$ of $f\in H$ as
\[
\RP (f) := E_{\bold{x}}[\RT(f)]
\]
where
\[
\RT(f)=\sum_{i,j:r_i>r_j}D(r_i,r_j)L(f(x_i)-f(x_j))
\]
and $D(r_i,r_j)$ is some positive weight function, which we take for simplicity to be
$1/|\mathcal{P}|$. (This uniform weight is the setting we have taken in the
main body of the paper.)
The smallest possible $L$-risk in $H$ is denoted
\[
\RP:= \inf_{f\in H} \RP(f).
\]
The {\it regularized} $L$-$risk$ is
\begin{equation}\label{regularize}
\Rp^{\text{reg}}(f):=\lambda \lVert f\rVert^2+\mathcal{R}_{L,P}(f),
\end{equation}
 $\lambda >0$.

For simplicity we assume the preference pair set $\mathcal{P}$ contains all pairs over these $n$ samples.
Let $g_{\bold{x},\lambda}$  be the optimal solution to the rank-AD minimization step.
Setting $\lambda = 1/2C$ and replacing $C$ with $\lambda$ in the rank-SVM step, we have:
\begin{equation}\label{eq_emp_solution}
  g_{\bold{x},\lambda} = \arg \min_{f\in H} \RT(f) + \lambda || f ||^2
\end{equation}

Let $\mathcal{H}_n$ denote a ball of radius $O(1/\sqrt{\lambda_n})$
in $H$. Let $C_k:= \sup_{x,t}|k(x,t)|$ with $k$ the rbf kernel associated
to $H$. Given $\epsilon>0$, we let
$N(\mathcal{H},\epsilon/4C_k)$ be the covering number
of $\mathcal{H}$ by disks of radius $\epsilon/4C_k$ .
We first show that with appropriately chosen $\lambda$, as $n\rightarrow\infty$,
$g_{\bold{x},\lambda}$ is consistent in the following sense.

\begin{lem}\label{thm_convergence_surrogate}
Let $\lambda_n$ be appropriately chosen such that $\lambda_n\rightarrow 0$ and $\frac{\log N(\h_n,\epsilon/4C_k)}{n\lambda_n} \to 0$, as $n\rightarrow\infty$. Then we have
\[
    E_{\bold{x}}[ \RT(g_{\bold{x},\lambda})] \rightarrow  \RP=\min_{f\in H}\RP(f), \;\;\; n\to \infty.
\]
\end{lem}
\begin{Proof}
Let us outline the argument. In \cite{ref:Steinwart2001}, the author shows that there exists a $f_{P,\lambda}\in H$
minimizing (\ref{regularize}):
\begin{lemma}\label{convex}
For all Borel probability measures $P$ on $X\times X$ and
all $\lambda >0$, there is an $f_{P,\lambda} \in H$ with
\[
\Rp^{\text{reg}}(f_{P,\lambda}) = \inf_{f\in H} \Rp^{\text{reg}}(f)
\]
such that $\lVert f_{P,\lambda} \rVert =O(1/\sqrt{\lambda})$.
\end{lemma}

Next, a simple argument shows that
\[
\lim_{\lambda\to 0} \Rp^{\text{reg}}(f_{P,\lambda})= \RP.
\]

Finally, we will need a concentration inequality to relate the $L$-risk
of $f_{P,\lambda}$ with the empirical $L$-risk of $f_{T,\lambda}$.
We then derive consistency using the following argument:
\begin{align*}
\RP(f_{T,\lambda_n})
& \leq  \lambda_n \lVert f_{T,\lambda_n}\rVert^2+\mathcal{R}_{L,P}(f_{T,\lambda_n})\\
& \leq \lambda_n \lVert f_{T,\lambda_n}\rVert^2+\mathcal{R}_{L,T}(f_{T,\lambda_n})+\delta/3  \\
& \leq  \lambda_n \lVert f_{P,\lambda_n}\rVert^2+\mathcal{R}_{L,T}(f_{P,\lambda_n})+\delta/3 \\
& \leq  \lambda_n \lVert f_{P,\lambda_n}\rVert^2+\mathcal{R}_{L,P}(f_{P,\lambda_n})+2\delta/3  \\
& \leq  \mathcal{R}_{L,P}+\delta
\end{align*}

where $\lambda_n$ is an appropriately chosen sequence $\to 0$,
and $n$ is large enough. The second and fourth inequality hold due to Concentration Inequalities, and the last one holds since $\lim_{\lambda \to 0} \Rp^{\text{reg}}(f_{P,\lambda})=\mathcal{R}_{L,P}$.

We now prove the appropriate concentration inequality \cite{ref:Smale2001}.
Recall $H$ is an RKHS with
smooth kernel $k$; thus the inclusion $I_{k}: H\to C(X)$ is compact, where
$C(X)$ is given the $\lVert \cdot \rVert_{\infty}$-topology. That is, the
``hypothesis space'' $\mathcal{H}:= \overline{I_k(B_R)}$ is compact in $C(X)$,
where $B_R$ denotes the ball of radius $R$ in $H$. We denote by $N(\mathcal{H},
\epsilon)$ the covering number of $\mathcal{H}$ with disks of radius $\epsilon$.
We prove the following inequality:

\begin{lemma}
For any probability distribution $P$ on $X\times X$,
\begin{equation*}
P^{\e}\{T\in (X\times X)^{\e}:\sup_{f\in \h} | \RT(f)-\RP(f)| \geq \epsilon \}\leq
2N(\h,\epsilon/4C_k)\exp\left(
\frac{-\epsilon^2n}{2(1+2\sqrt{C_k}R)^2}\right),
\end{equation*}

where $C_k := \sup_{x,t}|k(x,t)|$.
\end{lemma}

\begin{Proof}
Since $\h$ is compact, it has a finite covering number. Now suppose $\h=
D_1\cup \cdots \cup D_{\ell}$ is any finite covering of $\h$. Then
\begin{equation*}
\text{Prob}\{\sup_{f\in \h} | \RT(f)-\RP(f)| \geq \epsilon \} \leq  \sum_{j=1}^{\ell}
\text{Prob}\{\sup_{f\in D_j} | \RT(f)-\RP(f)| \geq \epsilon \}
\end{equation*}
so we restrict attention to a disk $D$ in $\h$ of appropriate radius $\epsilon$.

Suppose $\lVert f-g \rVert_{\infty}\leq \epsilon$. We want to show that the
difference
\[
|(\RT(f)-\RP(f))-(\RT(g)-\RP(g))|
\]
is also small. Rewrite this quantity as
\[
 |(\RT(f)-\RT(g))-E_{\bold{x}}[\RT(g)-\RT(f)]|.
 \]
 Since $\lVert f-g \rVert_{\infty}\leq \epsilon$, for $\epsilon$ small enough we have
\begin{align*}
\max\{0,1-(f(x_i)-f(x_j))\}-\max\{0,1-(g(x_i)-g(x_j))\} & = \max\{0,(g(x_i)-g(x_j)-f(x_i)+f(x_j))\} \\
& = \text{max}\{0,\langle g-f, \phi(x_i)-\phi(x_j)\rangle\}.
\end{align*}
Here $\phi:X\to H$ is the feature map, $\phi(x):=k(x,\cdot)$.
 Combining this with the Cauchy-Schwarz inequality, we have
\begin{eqnarray*}
 |(\RT(f)-\RT(g))-E_{\bold{x}}[\RT(g)-\RT(f)]| & \leq  \frac{2}{n^2}(2n^2\lVert f-g \rVert_{\infty}C_k) & \leq 4C_k\epsilon,
\end{eqnarray*}
where $C_k:= \sup_{x,t} |k(x,t)|$. From this inequality it follows that
\begin{equation*}
|\RT(f)-\RP(f)| \geq (4C_k+1)\epsilon  \implies |(\RT(g)-\RP(g))| \geq \epsilon.
\end{equation*}
We thus choose to cover $\h$ with disks of radius $\epsilon/4C_k$, centered at
$f_1,\dots,f_{\ell}$. Here $\ell= N(\h,\epsilon/4C_k)$ is the covering number
for this particular radius. We then have
\begin{equation*}
\sup_{f\in D_j}|(\RT(f)-\RP(f))|\geq 2\epsilon  \implies |(\RT(f_j)-\RP(f_j))|\geq \epsilon.
\end{equation*}
Therefore,
\begin{equation*}
 \text{Prob}\{\sup_{f\in \h} | \RT(f)-\RP(f)| \geq 2\epsilon \}  \leq
 \sum_{j=1}^n \text{Prob}\{ | \RT(f_j)-\RP(f_j)| \geq \epsilon \}
 \end{equation*}
The probabilities on the RHS can be bounded using McDiarmid's inequality.

Define the random variable $g(x_1,\dots,x_n) :=\mathcal{R}_{L,T}(f)$, for fixed $f\in H$.
We need to verify that $g$ has bounded differences. If we change one of the
variables, $x_i$, in $g$ to $x_i'$, then at most $n$ summands will change:
\begin{align*}
|g(x_1,\dots,x_i,\dots,x_n)-g(x_1,\dots,x_i',\dots,x_n)|
& \leq \frac{1}{n^2}2n\sup _{x,y} |1-(f(x)-f(y))| \\
& \leq \frac{2}{n}+\frac{2}{n}\sup_{x,y}|f(x)-f(y)|\\
& \leq \frac{2}{n}+\frac{4}{n}\sqrt{C_k}\lVert f \rVert.
\end{align*}
Using that $\sup_{f\in \mathcal{H}}\lVert f \rVert\leq R$,
McDiarmid's inequality thus gives
\begin{equation*}
\text{Prob}\{\sup_{f\in \h} | \RT(f)-\RP(f)| \geq \epsilon \}
 \leq 2N(\h,\epsilon/4C_k)\exp\left(
\frac{-\epsilon^2n}{2(1+2\sqrt{C_k}R)^2}\right).
\end{equation*}
\end{Proof}

We are now ready to prove Theorem 2. 
Take $R=\lVert f_{P,\lambda} \rVert$ and apply this result to $f_{P,\lambda}$:
\begin{equation*}
\text{Prob}\{| \RT(f_{P,\lambda})-\RP(f_{P,\lambda})| \geq \epsilon \}  \leq
2N(\h,\epsilon/4C_k)\exp\left(
\frac{-\epsilon^2n}{2(1+2\sqrt{C_k}\lVert f_{P,\lambda} \rVert)^2}\right).
\end{equation*}
Since $\lVert f_{P,\lambda_n} \rVert =O(1/\sqrt{\lambda_n})$, the RHS converges to 0
so long as $\dfrac{n\lambda_n}{\log N(\h,\epsilon/4C_k)} \to \infty$ as $n\to \infty$.
This completes the proof of Theorem 2. 
\end{Proof}

We now establish that under mild conditions on the surrogate loss function, the solution minimizing the expected surrogate loss will asymptotically recover the correct preference relationships given by the density $f$.
\begin{lem}\label{thm_surrogate_condition}
Let $L$ be a non-negative, non-increasing convex surrogate loss function
that is differentiable at zero and satisfies $L'(0)<0$. If
\begin{equation*}
  g^* = \arg \min_{g\in H} \mathbb{E}_{\bold{x}} \left[ \RT(g) \right],
\end{equation*}
then $g^*$ will correctly rank the samples according to their density, i.e.
$\forall x_i\neq x_j, f(x_i)> f(x_j) \implies g^*(x_i)>g^*(x_j)$.
Assume the input preference pairs satisfy: $\mathcal{P}=\{(x_i,x_j):\, f(x_i)>f(x_j)\}$, where $\bold{x} = \{x_1,\ldots,x_n\}$ is drawn i.i.d. from distribution $f$.
Let $\ell$ be some convex surrogate loss function that satisfies: (1) $\ell$ is non-negative and non-increasing; (2) $\ell$ is differentiable and $\ell'(0)<0$.
Then the optimal solution:
$g^*$, will correctly rank the samples according to $f$, i.e. $g^*(x_i)>g^*(x_j)$, $\forall x_i\neq x_j, f(x_i)>f(x_j)$, .
\end{lem}

The hinge-loss satisfies the conditions in the above theorem.
Combining Theorem \ref{thm_convergence_surrogate} and \ref{thm_surrogate_condition}, we establish that asymptotically, the rank-SVM step yields a ranker that preserves the preference relationship on nominal samples given by the nominal density $f$.

\begin{Proof} Our proof follows similar lines of Theorem 4 in \cite{ref:RDPS2012}.
Assume that $g(x_i) < g(x_j)$, and define a function $g'$ such that
$g'(x_i)=g(x_j)$, $g'(x_j)=g(x_i)$, and $g'(x_k)=g(x_k)$ for all $k\neq i,j$.
We have $\RP(g')-\RP(g)=E_{\bold{x}}(A(\bold{x}))$, where

\begin{eqnarray*}
 A(\bold{x})
   = \sum_{k : r_j<r_i<r_k}[D(r_k,r_j)-D(r_k,r_i)]  [L(g(x_k)-g(x_i))-L(g(x_k)-g(x_j))]
  \\ +  \sum_{k : r_j<r_k<r_i}D(r_i,r_k)[L(g(x_j)-g(x_k))-L(g(x_i)-g(x_k))]
  \\ +  \sum_{k : r_j<r_k<r_i}D(r_k,r_j)[L(g(x_k)-g(x_i))-L(g(x_k)-g(x_j))]
  \\ +  \sum_{k : r_j<r_i<r_k}[D(r_k,r_j)-D(r_k,r_i)][L(g(x_k)-g(x_i))-L(g(x_k)-g(x_j))]
  \\ +  \sum_{k : r_j<r_i<r_k}[D(r_i,r_k)-D(r_j,r_k)][L(g(x_j)-g(x_k))-L(g(x_i)-g(x_k))]
  \\ +  (L(g(x_j)-g(x_i))-L(g(x_i)-g(x_j)))D(r_i,r_j).
  \end{eqnarray*}
  Using the requirements of the weight function $D$ and the assumption that $L$
  is non-increasing and non-negative, we see that all six sums in the above
  equation for $A(\bold{x})$ are negative. Thus $A(\bold{x})<0$, so
  $\RP(g')-\RP(g)=E_{\bold{x}}(A(\bold{x}))<0$, contradicting the minimality
  of $g$. Therefore $g(x_i)\geq g(x_j)$.

  Now we assume that $g(x_i)=g(x_j)=g_0$. Since $\RP(g)=\inf_{h\in H}\RP(h)$,
  we have $\left. \dfrac{\partial{\ell_L(g;x)}}{\partial{g(x_i)}}\right|_{g_0}=A=0,$ and
  $\left. \dfrac{\partial{\ell_L(g;x)}}{\partial{g(x_j)}}\right|_{g_0}=B=0$, where
  \begin{eqnarray*}
  A=\sum_{k : r_j < r_i < r_k} D(r_k, r_i) [ -L'(g(x_k)-g_0)]+
    \sum_{k : r_j < r_k< r_i} D(r_i, r_k) L'(g_0-g(x_k)) +\\
  \sum_{k : r_k < r_j < r_i} D(r_i, r_k)  L'(g_0-g(x_k))+D(r_i,r_j)[-L'(0)].
  \end{eqnarray*}
  \begin{eqnarray*}
  B=\sum_{k : r_j < r_i < r_k} D(r_k, r_j) [ -L'(g(x_k)-g_0)]+
  \sum_{k : r_j < r_k< r_i} D(r_k, r_j) L'(g_0-g(x_k)) +\\
  \sum_{k : r_k < r_j < r_i} D(r_j, r_k)  L'(g_0-g(x_k))+D(r_i,r_j)[-L'(0)].
  \end{eqnarray*}
  However, using $L'(0)<0$ and the requirements of $D$ we have
  \[
  A-B\leq 2L'(0)D(r_i,r_j)<0,
  \]
  contradicting $A=B=0$.
\end{Proof}

The following lemma completes the proof of Theorem 2: 
\begin{lem}\label{lem_order_consistency}
Assume $G$ is any function that gives the same order relationship as the density: $G(x_i)>G(x_j)$, $\forall x_i\neq x_j$ such that $f(x_i)>f(x_j)$. Then
\begin{equation}\label{eq:rank}
 \frac{1}{n} \sum_{i=1}^n \textbf{1}_{\{G(x_i) \leq G(\eta)\}} \rightarrow p(\eta).
\end{equation}
\end{lem}

\section*{Proof of Theorem 3} 

To prove Theorem 3 
we need the following lemma \cite{ref:vapnik1979}:

\begin{lemma}\label{vapnik}
Let $\X$ be a set and $S$ a system of sets in $\X$, and $P$ a probability
measure on $S$. For $\x\in \X^{n}$ and $A\in S$, define $\nu_{\x}(A):= |\x\cap A|/n$.
If $n>2/\epsilon$, then
\begin{equation*}
P^{n}\left\{ \x : \sup_{A\in S} |\nu_{\x}(A)- P(A)|>\epsilon\right\}  \leq
2P^{2n}\left\{ \x\x' : \sup_{A\in S} |\nu_{\x}(A)-\nu_{\x'}(A)|>\epsilon/2\right\}.
\end{equation*}
\end{lemma}

\begin{Proof}
Consider the event
\begin{equation*}
J:=  \Biggl\{  \x \in \X^n : \exists f\in \F, P\{ x : f(x) < f^{(m)} - 2\gamma \}  > \frac{m-1}{n}+\epsilon \Biggr\}.
\end{equation*}

We must show that $P^n(J)\leq \delta$ for $\epsilon = \epsilon(n,k,\delta)$.
Fix $k$ and apply lemma \ref{vapnik} with
\[
A=\{ x : f(x) < f^{(m)} - 2\gamma \}
\]
with $\gamma$ small enough so that
\[
\nu_{\x}(A)= |\{x_j\in \x : f(x_j) < f^{(m)} - 2\gamma\}|/n=
 \frac{m-1}{n}.
 \]
We obtain
\[
P^n(J)\leq 2P^{2n} \Biggl\{ \x\x' :  \exists f\in \F, | \{ x_j' \in \x' : f(x_j')  < f^{(m)} - 2\gamma\}| > \epsilon n/2 \Biggr\}.
\]
The remaining portion of the proof follows as Theorem 12 in \cite{ref:oc_svm2001}.
\end{Proof}

\bibliographystyle{unsrt}
\bibliography{ranksvm}

\begin{thebibliography}{30}
\providecommand{\natexlab}[1]{#1}
\providecommand{\url}[1]{\texttt{#1}}
\expandafter\ifx\csname urlstyle\endcsname\relax
  \providecommand{\doi}[1]{doi: #1}\else
  \providecommand{\doi}{doi: \begingroup \urlstyle{rm}\Url}\fi

\bibitem[Basseville et~al.(1993)Basseville, Nikiforov, et~al.]{ref:para_1993}
M.~Basseville, I.V. Nikiforov, et~al.
\newblock \emph{Detection of abrupt changes: theory and application}, volume
  104.
\newblock Prentice Hall Englewood Cliffs, NJ, 1993.

\bibitem[Bay and Schwabacher(2003)]{orca}
Stephen~D. Bay and Mark Schwabacher.
\newblock Mining distance-based outliers in near linear time with randomization
  and a simple pruning rule.
\newblock In \emph{Proceedings of the Ninth ACM SIGKDD International Conference
  on Knowledge Discovery and Data Mining}, KDD '03, pages 29--38, New York, NY,
  USA, 2003. ACM.
\newblock ISBN 1-58113-737-0.
\newblock \doi{10.1145/956750.956758}.
\newblock URL \url{http://doi.acm.org/10.1145/956750.956758}.

\bibitem[Bhaduri et~al.(2011)Bhaduri, Matthews, and Giannella]{ref:Bhaduri11}
K.~Bhaduri, B.~L. Matthews, and C.~R. Giannella.
\newblock Algorithms for speeding up distance-based outlier detection.
\newblock In \emph{ACM SIGKDD}, pages 859--867, 2011.

\bibitem[Chandola et~al.(2009)Chandola, Banerjee, and
  Kumar]{ref:anomaly_detection_survey}
V.~Chandola, A.~Banerjee, and V.~Kumar.
\newblock Anomaly detection: A survey.
\newblock \emph{ACM Comput. Surv.}, 41\penalty0 (3):\penalty0 15:1--15:58, July
  2009.
\newblock ISSN 0360-0300.
\newblock \doi{10.1145/1541880.1541882}.
\newblock URL \url{http://doi.acm.org/10.1145/1541880.1541882}.

\bibitem[Chang and Lin(2011)]{ref:libsvm}
C.~Chang and C.~Lin.
\newblock Libsvm: A library for support vector machines.
\newblock \emph{ACM Trans. Intell. Syst. Technol.}, 2\penalty0 (3):\penalty0
  27:1--27:27, May 2011.
\newblock ISSN 2157-6904.
\newblock \doi{10.1145/1961189.1961199}.
\newblock URL \url{http://doi.acm.org/10.1145/1961189.1961199}.

\bibitem[Chapelle and Keerthi(2010)]{ref:ranksvm_chapelle}
O.~Chapelle and S.~S. Keerthi.
\newblock Efficient algorithms for ranking with svms.
\newblock In \emph{Information Retrieval}, volume~81, pages 201--215, 2010.

\bibitem[Cucker and Smale(2001)]{ref:Smale2001}
F.~Cucker and S.~Smale.
\newblock On the mathematical foundations of learning.
\newblock In \emph{Bull. Amer. Math. Soc.}, pages 1--49, 2001.

\bibitem[Cuevas and Rodriguez-Casal(2003)]{ref:Cuevas2003}
A.~Cuevas and A.~Rodriguez-Casal.
\newblock Set estimation: An overview and some recent developments.
\newblock In \emph{Recent advances and trends in nonparametric statistics},
  pages 251--264, 2003.

\bibitem[Frank and Asuncion(2010)]{ref:UCI}
A.~Frank and A.~Asuncion.
\newblock {UCI} machine learning repository.
\newblock \\\url {http://archive.ics.uci.edu/ml}, 2010.

\bibitem[Hero(2006)]{ref:GEM_2006}
A.O. Hero.
\newblock Geometric entropy minimization (gem) for anomaly detection and
  localization.
\newblock In \emph{Neural Information Processing Systems Conference},
  volume~19, 2006.

\bibitem[Hodge and Austin(2004)]{ref:AD_survey_hodge}
V.~Hodge and J.~Austin.
\newblock A survey of outlier detection methodologies.
\newblock In \emph{Artificial Intelligence Review}, volume~22, pages 85--126,
  2004.

\bibitem[Joachims(2002)]{ref:ranksvm}
T.~Joachims.
\newblock Optimizing search engines using clickthrough data.
\newblock In \emph{Proceedings of the eighth ACM SIGKDD international
  conference on Knowledge discovery and data mining}, KDD '02, pages 133--142,
  New York, NY, USA, 2002. ACM.
\newblock ISBN 1-58113-567-X.
\newblock \doi{10.1145/775047.775067}.
\newblock URL \url{http://doi.acm.org/10.1145/775047.775067}.

\bibitem[Kulis and Grauman(2009)]{klsh}
Brian Kulis and Kristen Grauman.
\newblock Kernelized locality-sensitive hashing for scalable image search.
\newblock In \emph{IEEE International Conference on Computer Vision (ICCV},
  2009.

\bibitem[Lan et~al.(2012)Lan, Guo, Cheng, and Liu]{ref:RDPS2012}
Y.~Lan, J.~Guo, X.~Cheng, and T.~Liu.
\newblock Statistical consistency of ranking methods in a rank-differentiable
  probability space.
\newblock In \emph{Advances in Neural Information Processing Systems}, pages
  1241--1249, 2012.

\bibitem[Langford(2005)]{langford05}
John Langford.
\newblock Tutorial on practical prediction theory for classification.
\newblock \emph{J. Mach. Learn. Res.}, 6:\penalty0 273--306, December 2005.

\bibitem[Liu et~al.(2008)Liu, Ting, and Zhou]{ref:isolation_forest}
F.~T. Liu, K.~M. Ting, and Z.~Zhou.
\newblock Isolation forest.
\newblock In \emph{Proceedings of the 2008 Eighth IEEE International Conference
  on Data Mining}, pages 413--422, 2008.

\bibitem[Nunez-Garcia et~al.(2003)Nunez-Garcia, Kutalik, K.-H.Cho, and
  Wolkenhauer]{ref:levelset_mvset_2003}
J.~Nunez-Garcia, Z.~Kutalik, K.-H.Cho, and O.~Wolkenhauer.
\newblock Level sets and minimum volume sets of probability density functions.
\newblock In \emph{Approximate Reasoning}, volume~34, pages 25--47, 2003.

\bibitem[Orair et~al.(2010)Orair, Teixeira, Meira, Wang, and
  Parthasarathy]{orair}
G.~H. Orair, Carlos H.~C. Teixeira, Wagner Meira, Jr., Ye~Wang, and Srinivasan
  Parthasarathy.
\newblock Distance-based outlier detection: Consolidation and renewed bearing.
\newblock \emph{Proc. VLDB Endow.}, 3\penalty0 (1-2):\penalty0 1469--1480,
  September 2010.
\newblock ISSN 2150-8097.
\newblock \doi{10.14778/1920841.1921021}.
\newblock URL \url{http://dx.doi.org/10.14778/1920841.1921021}.

\bibitem[Park et~al.(2010)Park, Huang, and Ding]{ref:MV_2010}
C.~Park, J.~Z. Huang, and Y.~Ding.
\newblock A computable plug-in estimator of minimum volume sets for novelty
  detection.
\newblock \emph{Operations Research}, pages 1469--1480, 2010.

\bibitem[Qian and Saligrama(2012)]{ref:Jing2012}
J.~Qian and V.~Saligrama.
\newblock New statistic in p-value estimation for anomaly detection.
\newblock In \emph{Statistical Signal Processing Workshop, IEEE}, pages 393
  --396, Aug. 2012.
\newblock \doi{10.1109/SSP.2012.6319713}.

\bibitem[Rocke and Woodruff(1996)]{ref:Rocke96}
D.~M. Rocke and D.~L. Woodruff.
\newblock Identification of outliers in multivariate data.
\newblock In \emph{Journal of the American Statistical Association}, pages
  1047--1061, 1996.

\bibitem[Sch{\"o}lkopf et~al.(2001)Sch{\"o}lkopf, Platt, Shawe-Taylor, Smola,
  and Williamson]{ref:oc_svm2001}
B.~Sch{\"o}lkopf, J.C. Platt, J.~Shawe-Taylor, A.J. Smola, and R.C. Williamson.
\newblock Estimating the support of a high-dimensional distribution.
\newblock \emph{Neural computation}, 13\penalty0 (7):\penalty0 1443--1471,
  2001.

\bibitem[Scott and Nowak(2006)]{ref:MV_2006}
C.D. Scott and R.D. Nowak.
\newblock Learning minimum volume sets.
\newblock \emph{The Journal of Machine Learning Research}, 7:\penalty0
  665--704, 2006.

\bibitem[Sricharan and Hero(2011)]{ref:knn_2011}
K.~Sricharan and A.~O. Hero.
\newblock Efficient anomaly detection using bipartite k-nn graphs.
\newblock In \emph{Neural Information Processing Systems}, 2011.

\bibitem[Steinwart(2001)]{ref:Steinwart2001}
I.~Steinwart.
\newblock Consistency of support vector machines and other regularized kernel
  machines.
\newblock In \emph{IEEE Trans. Inform. Theory}, pages 67--93, 2001.

\bibitem[Ting et~al.(2010)Ting, Zhou, Liu, and Tan]{ref:massAD}
K.~M. Ting, G.~Zhou, F.~T. Liu, and J.~S.~C. Tan.
\newblock Mass estimation and its applications.
\newblock In \emph{Proceedings of the 16th ACM SIGKDD international conference
  on Knowledge discovery and data mining}, KDD '10, pages 989--998, New York,
  NY, USA, 2010. ACM.

\bibitem[Vapnik(1979)]{ref:vapnik1979}
V.~Vapnik.
\newblock \emph{Estimation of Dependences Based on Empirical Data [in
  Russian]}.
\newblock English translation: Springer Verlag, New York, 1982, 1979.

\bibitem[Wang et~al.(2011)Wang, Parthasarathy, and Tatikonda]{lsh}
Ye~Wang, Srinivasan Parthasarathy, and Shirish Tatikonda.
\newblock Locality sensitive outlier detection: A ranking driven approach.
\newblock In Serge Abiteboul, Klemens Bšhm, Christoph Koch, and Kian-Lee Tan,
  editors, \emph{ICDE}, pages 410--421. IEEE Computer Society, 2011.
\newblock ISBN 978-1-4244-8958-9.
\newblock URL
  \url{http://dblp.uni-trier.de/db/conf/icde/icde2011.html#WangPT11}.

\bibitem[Yamanishi et~al.(2000)Yamanishi, Takeuchi, Williams, and
  Milne]{ref:Yamanishi00}
K.~Yamanishi, J.-I. Takeuchi, G.~Williams, and P.~Milne.
\newblock Online unsupervised outlier detection using finite mixtures with
  discounting learning algorithms.
\newblock In \emph{Proceedings of the ACM SIGKDD}, pages 320--324, 2000.

\bibitem[Zhao and Saligrama(2009)]{ref:Manqi2009}
M.~Zhao and V.~Saligrama.
\newblock Anomaly detection with score functions based on nearest neighbor
  graphs.
\newblock In \emph{Neural Information Processing Systems Conference},
  volume~22, 2009.

\end{thebibliography}


\end{document}